\documentclass[sigconf]{acmart}

\AtBeginDocument{%
  \providecommand\BibTeX{{%
    \normalfont B\kern-0.5em{\scshape i\kern-0.25em b}\kern-0.8em\TeX}}}

\copyrightyear{2021} 
\acmYear{2021} 
\setcopyright{acmcopyright}
\acmConference[CIKM '21]{Proceedings of the 30th ACM International Conference on Information and Knowledge Management}{November 1--5, 2021}{Virtual Event, QLD, Australia}
\acmBooktitle{Proceedings of the 30th ACM International Conference on Information and Knowledge Management (CIKM '21), November 1--5, 2021, Virtual Event, QLD, Australia}
\acmPrice{15.00}
\acmDOI{10.1145/3459637.3482447}
\acmISBN{978-1-4503-8446-9/21/11}

\usepackage{multirow}
\usepackage{color}
\usepackage{graphics}

\usepackage{tikz}
\usetikzlibrary{fit, shapes, shapes.arrows, calc}
\usetikzlibrary{arrows.meta}
\usepackage{pgfplots}
\usepgfplotslibrary{groupplots}
\usepackage{xcolor}

\definecolor{aluminum}{RGB}{153,153,153}
\definecolor{platinum}{RGB}{228,228,228}
\definecolor{bgc}{RGB}{245,245,245}
\definecolor{gallery}{RGB}{240,240,240}
\definecolor{tuatara}{RGB}{67, 67, 67}
\definecolor{flamingo}{RGB}{237, 88, 85}
\definecolor{salmon}{RGB}{242,131,107}
\definecolor{free_speech_aquamarine}{RGB}{0, 156, 114}

\definecolor{bb}{HTML}{95e1d3}
\definecolor{gg}{HTML}{c7ffd8}
\definecolor{yy}{HTML}{f0c38e}

\definecolor{rr}{HTML}{f38181}

\usepackage{hyperref}
\usepackage{subfigure}

\begin{document}

\fancyhead{}

\title{LEReg: Empower Graph Neural Networks with Local Energy Regularization}

\author{Xiaojun Ma}
\email{mxj@pku.edu.cn}
\affiliation{
\institution{Key Laboratory of Machine Perception, Ministry of Education, Peking University}
\city{Beijing}
\country{China}
}

\author{Hanyue Chen}
\email{
     ypchy@pku.edu.cn
}
\affiliation{
\institution{Yuanpei College, Peking 
University}
\city{Beijing}
\country{China}
}

\author{Guojie Song}
\authornote{Corresponding author.}
\email{
    gjsong@pku.edu.cn
}
\affiliation{
\institution{Key Laboratory of Machine Perception, Ministry of Education, Peking University}
\city{Beijing}
\country{China}
}

\begin{abstract}

Researches on analyzing graphs with Graph Neural Networks (GNNs) have been receiving more and more attention because of the great expressive power of graphs. GNNs map the adjacency matrix  and node features to node representations by message passing through edges on each convolution layer. However, the message passed through GNNs is not always beneficial for all parts in a graph. Specifically, as the data distribution is different over the graph, the receptive field (the farthest nodes that a node can obtain information from) needed to gather information is also different. Existing GNNs treat all parts of the graph uniformly, which makes it difficult to adaptively pass the most informative message for each unique part. To solve this problem, we propose two regularization terms that consider message passing locally: (1) Intra-Energy Reg and (2) Inter-Energy Reg. Through experiments and theoretical discussion, we first show that the speed of smoothing of different parts varies enormously and the topology of each part affects the way of smoothing. With Intra-Energy Reg, we strengthen the message passing within each part, which is beneficial for getting more useful information. With Inter-Energy Reg, we improve the ability of GNNs to distinguish different nodes. With the proposed two regularization terms, GNNs are able to filter the most useful information adaptively, learn more robustly and gain higher expressiveness. Moreover, the proposed LEReg can be easily applied to other GNN models with plug-and-play characteristics. Extensive experiments on several benchmarks verify that GNNs with LEReg outperform or match the state-of-the-art methods. The effectiveness and efficiency are also empirically visualized with elaborate experiments.

\end{abstract}



\maketitle

\section{Introduction}
\label{sec:introduction}

Recently, there has been a surge of researches in learning the graph data with Graph Neural Networks (GNNs), such as molecules \cite{gcpn, LanczosNet}, social networks \cite{GovernCascade}, and biological interactions \cite{daggnn}.
Starting with the huge success of Graph Convolutional Networks (GCNs) on semi-supervised classification, which stands out as one of the most powerful tools and the key operation, variants of GNNs \cite{gat,sgc,gin} have been proposed to achieve more powerful representations. 

With the development of GNNs, many methods are proposed to break the ceiling of the expressiveness of GNNs by various means. 
One idea is to utilize multi-scale information with skip connections and residual layers \cite{resgcn, StrongerGCN, gcnii}. 
Meanwhile, denoising in graphs has received incremental attention. \cite{madreg} designs a regularization term based on node view and modified the graph topology to reduce the noise. PairNorm and GroupNorm \cite{pairnorm, groupnorm} apply normalization on different scales. 
\cite{sgc,daggnn} decouple the process of transformation and propagation in GNNs, which helps to prevent over-fitting and over-smoothing by simplifying the model. 

Although the above methods push GNNs towards higher stages, they either treat nodes on the whole graph in the same manner or do not jointly consider features, graph topology, and ground truth labels. 
For most graph data in the real world, the topology and feature distribution vary among different parts, which enables GNN model to capture the relation between these two inputs and node labels in node classification tasks. 
We then ask a question: \textit{Is the impact of GNNs on different parts on the graph the same?}

\begin{figure}
    \centering
    \includegraphics[width=0.47\textwidth]{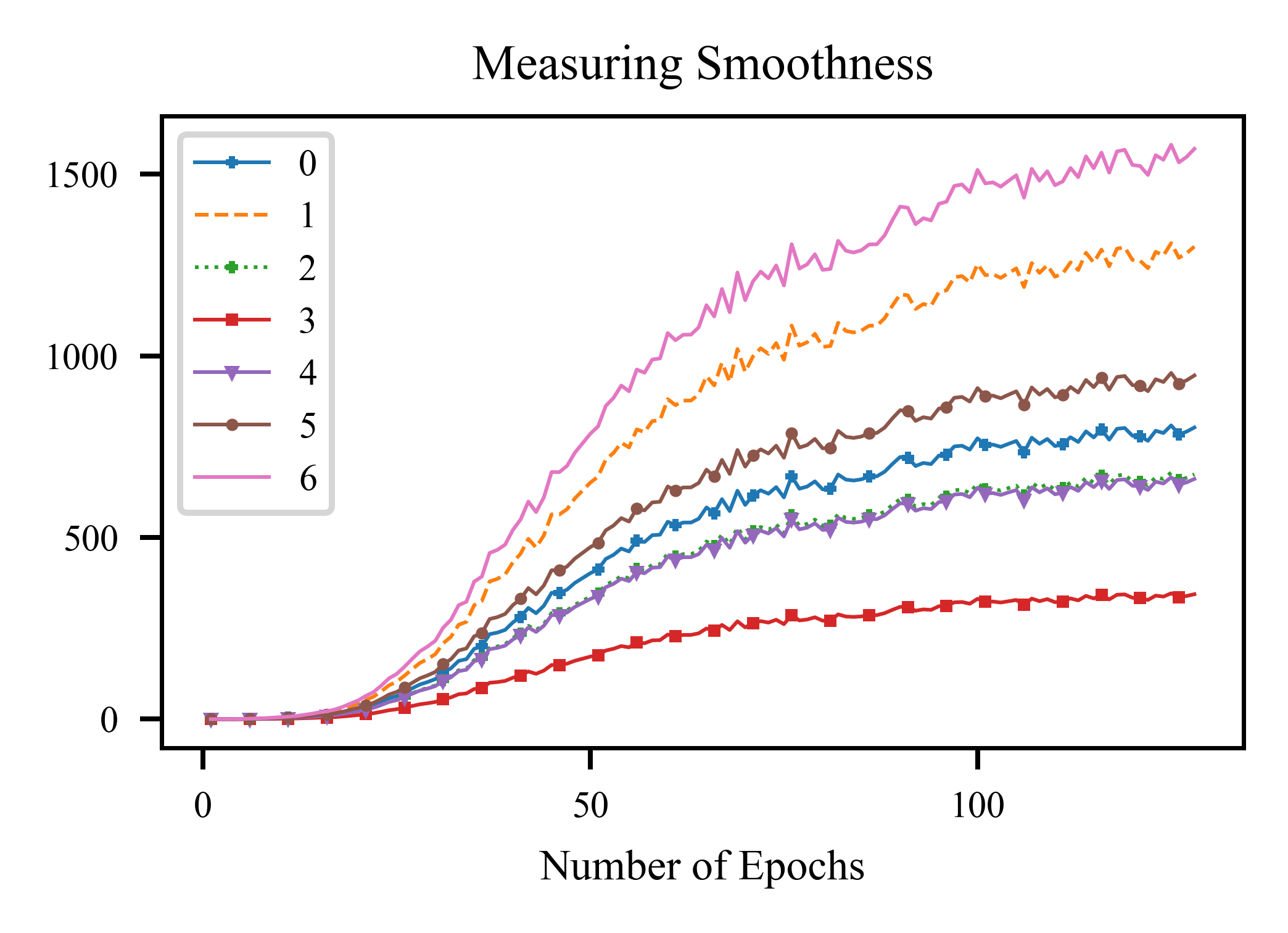}
    \caption{ 
    Evolution of smoothness over the training process on different classes. Each class is represented by a particular curve.
    }
    \label{fig:case_study}
\end{figure}

To answer this question, we conduct a simple experiment to show the evolution of smoothing degree for each class over the training process. 
A 2-layer GCN is adopted to train the citation dataset Cora \cite{citenet} for illustration. 
We separate the graph by ground truth node labels. For convenience, the smoothness of class $c_k$ is defined as $ \mathbb{E}_{i,j \in c_k} \left\| H^{(2)}_i - H^{(2)}_j  \right\|^2$ (notations refer to \autoref{sec:model}). 
As shown in \autoref{fig:case_study}, as the training goes on, the gap of smoothing degree between different sub-graphs becomes larger. The empirical results illustrate that GNN has different impacts on different sub-graphs.
This phenomenon verifies the intuitive discussion above. 

The above discussion and empirical results motivate us to conduct a specific graph convolution for each sub-graph, because the smoothing degree within a class is strongly related to the result of node classification. 
Besides, the convergence balance between classes benefits the training.    
However, training a specific GNN model for each sub-graph has the problem of high computation complexity and would lose the information between sub-graphs. 
Instead, we solve this problem with Local Energy Regularization (LEReg), which consists of Intra-Energy Reg and Inter-Energy Reg. 
For Intra-Energy Reg, we first generate a new graph by filtering the edges across classes with predicted class posterior probability. 
Minimizing the Dirichlet Energy defined on the new graph encourages the node embedding within a class to be close to each other. 
For Inter-Energy Reg, we view the graph at a high level by fusing the graph into a smaller graph, a node of which represents a class of the original graph. 
With the style of margin loss, minimizing the Inter-Energy Reg help to improve the power of GNNs to distinguish different classes.

We further provide the theoretical discussions that how LEReg boosts GNNs from three aspects in \autoref{sec:theoretical analysis}. 
First, LEReg balance the smoothness of different classes to boost GNNs. Second, with competitive efficiency, minimizing LEReg on each layer introduces extra information for GNNs. Acting as a graph structure learning mechanism, LEreg generates a cleaner graph by filtering the edges.
Last but not least, we relate LEReg to Deep GCN that minimizing the Intra-Energy Reg reaches the same output as applying infinite graph convolution on the sub-graphs, while LEReg is able to alleviate the problem of over-smoothing in deep GNNs. 

Our major contributions are summarized as follows:

\begin{enumerate}
    \item Capturing the unique information for each class is essential for GNNs. To our best knowledge, we make the first attempt to analyze how to boost GNNs with regularization in a local manner.
    \item We provide theoretical analysis to discuss how LEReg boosts GNNs with Intra-Energy Reg and Inter-Energy Reg from three aspects: (a) benefits of LEReg to balance the smoothness from the local aspect, (b) extra supervision information and joint learning of graph structure with node embeddings, (c)  connection between LEReg and deep GNNs.
    \item We conduct extensive experiments on six real-world datasets to validate the effectiveness of LEReg against the state-of-the-art methods.
\end{enumerate}

\section{Our Method: LEReg}
\label{sec:model}

In this section, we first introduce some notations related to GNNs. Then we give a detailed introduction of our proposed method LEReg, which consists of two components: Intra-Energy Reg and Inter-Energy Reg.

\subsection{Notations}

Given an attributed graph $G=(A, X)$ with vertex set $\mathcal{V}$ and  edge set $\mathcal{E}$, $A \in \mathbb{R}^{N \times N}$ denotes its adjacency matrix  and $X \in \mathbb{R}^{N \times N} $ denotes node feature matrix, where $N = |\mathcal{V}|$ is the number of nodes in the graph. $X = \left[x_1, \cdots, x_N \right]^T \in \mathbb{R}^{N \times F}$ is the node features. 
The node $v_i \in \mathcal{V}$ is associated with the $i$-th row of $X$, i.e. $x_i$ and a ground truth label $y_i$. 
Let $D$ denote the degree matrix, where $D_{i i}=\sum_{j=1}^{N} A_{i j}$ and $D_{i j}=0$ if $i \neq j$.

We introduce GNNs for node classification tasks, which classify nodes in a given graph to the right classes. We use a 2-layer GCN model $f: G \rightarrow \mathbb{R}^{N \times C}$  as an example of GNNs for convenience, where $C$ is the number of node classes. 
In details, $f$  can be formulated as $f(A,X) = \Tilde{A} \left(  \sigma \left( \Tilde{A} X W_0  \right) \right) W_1$, 
where $H$ is the size of hidden layer, $W_0 \in \mathbb{R}^{F \times H}$ and  $W_1 \in \mathbb{R}^{H \times C}$ are linear mapping matrices.
$\Tilde{A}=D^{-1} A$ is the normalized adjacency matrix
, and $\sigma$ is the activation function.
The output of $f$ is  $Z=f(A, X) \in \mathbb{R}^{N \times C}$. 
Denote $P = \{P_{ij}\}$ as the softmax of the output logits, where $P_{ij}=\frac{\exp \left(Z_{i j}\right)}{\sum_{k=1}^{C} \exp \left(Z_{i k}\right)}$ for $i=1, \ldots, N$ and $j=1, \ldots, C$.

\subsection{Intra-Energy Reg and Inter-Energy Reg}
To improve the expressive ability of GNN with local regularization rather than the global one, we need to split the graph into several sub-graphs. 
Intuitively, using the ground truth labels as the standard, the noise introduced by inter-class edges, which link nodes from different classes, can be reduced and the common signal among the same class can be well preserved. However, for semi-supervised classification with GNNs, only a few nodes with labels can be seen during training, which brings challenges for downstream tasks. Therefore, we utilize the predicted class probability to build the \textit{Intra-Energy Reg} and \textit{Inter-Energy Reg} instead. 

We can get the predicted class posterior probability $P_{i j}$ for $i=1, \cdots, N$ and $j=1, \cdots, C$ by applying softmax on the output logits. The largest predicted class posterior probability $\hat{y}_i = \arg \max_{j=1, \cdots, C} P_{i j}$ for node $v_i$ indicates that 
$v_i$ should be classified into category $\hat{y}_i$.
We get the new graph adjacency matrix $\hat{A}$ with only intra-class edges kept: 
\begin{equation}
    \label{eq:discreate mask}
    \hat{A}_{ij} = 
    \begin{cases}
    1 , & \text{if $\hat{y}_i = \hat{y}_j$ and $A_{ij}=1$ } \\
    0, & \text{otherwise} \\
    \end{cases}
\end{equation}

However, the procedure above is not differentiable, which limits the flexibility of training. We take a soft way to filter the edges. 
A simple but nontrivial way to separate the graph is to mask the adjacency matrix with a soft weight. 
Under discrete cases, we only consider the class with the largest predicted probability and ignore the rest class. 
But the predicted classes can not exactly matches the ground truth labels. 
Therefore, we use the expected value instead. Specifically, rather than assigning nodes into some classes then preserving the edges, we directly mask an edge with the probability that the corresponding two nodes are from the same class.
Given $\left\{ P_{i j} \right\}$, the probability that node $v_i$ and node $v_k$ are in the same class is
\begin{equation}
\label{eq:q probability}
    Q_{ik} = \sum_{j=1}^C  P_{ij} \cdot P_{kj} = \langle  p_i  , p_k \rangle ,
\end{equation}
where $p_i$ is the $i$-th row of $P$. The adjacency matrix after the soft mask is $\hat{A} = Q \odot A $, where $\odot$ means point-wise multiplication.

\paragraph{Intra-Energy Reg} The common pattern shared by nodes from the same class plays an important role in node classification tasks. Given a graph  $G=(A, X)$, the GNN model $f$ learns to map from node features and local topology to node labels, thereby trying to mine the connection between input and output. However, the graph is not completely clean and there exist lots of edges across classes, which brings noise into the learning of GNN. We propose a local graph Laplacian smoothing method to boost GNN by maintaining the common information pattern within the same class. Under global setting, the graph Laplacian smoothing is defined as 
\begin{equation}
    \label{eq:dirichlet_sum}
    \begin{aligned}
    E\left( H^{(L)} \right) 
    &= \frac{1}{2} \sum_{\left(v_{i}, v_{j}\right) \in \mathcal{E} } A_{i j}
    \left\|\frac{{H^{(L)}}_i}{\sqrt{d_{i}}}-\frac{{ H^{(L)} }_j}{\sqrt{d_{j}}}\right\|_{2}^{2} \\
    & = \operatorname{tr} \left( {H^{(L)} }^T L_{sym} H^{(L)} \right)
    ,
    \end{aligned}
\end{equation}
where $H^{(L)}$ is the output of the GNN model with $L$ convolution layers, $\operatorname{tr} (\cdot)$ is the trace operator and $L_{sym}:=I-D^{-\frac{1}{2}}AD^{-\frac{1}{2}}$ is the normalized Laplacian matrix. This sum in \autoref{eq:dirichlet_sum}  is also called Dirichlet Energy, which reflects the smoothness of the function $f$. 
Therefore, this regularization term is usually used in semi-supervised representation learning to provide graph structure information for a model. 
In this paper, we elaborately utilize this powerful Dirichlet Energy
\begin{equation}
    \label{eq:intra}
    \begin{aligned}
    E_{intra} \left( H^{(L)} \right) 
    & = \frac{1}{2} \sum_{k=1}^C 
    \sum_{ \left(v_{i}, v_{j}\right) \in \mathcal{E} } w_{ijk} A_{i j}
    \left\|\frac{ {H^{(L)}}_i }{\sqrt{ \hat{d}_{i}} }-\frac{{ H^{(L)} }_j}{\sqrt{\hat{d}_{j}}}\right\|_{2}^{2} \\
    & =  \frac{1}{2} \sum_{\left(v_{i}, v_{j}\right) \in \mathcal{E} } \hat{A}_{ i j }
    \left\| \frac{ {H^{(L)}}_i }{ \sqrt{\hat{d}_{i}} }-\frac{ { H^{(L)} }_j }{ \sqrt{ \hat{d}_{j} } } \right\|_{ 2 }^{ 2 } \\
    & = \operatorname{tr} \left( {H^{(L)} }^T \hat{L}_{sym} H^{(L)} \right)
    ,
    \end{aligned}
\end{equation}
where $w_{ijk} = P_{ik} \cdot P_{jk}$ is the probability that both node $v_i$ and $v_j$ belong to class $k$, and $\hat{L}_{sym}$ is the normalized Laplacian matrix of $\hat{A}$.
The Intra-Energy Reg is defined as $\mathcal{L}_{Intra} = E_{intra} \left( H^{(L)} \right) $.

\paragraph{Inter-Energy Reg} On the other hand, the distinctiveness of different classes has a powerful influence on the downstream task. The key to improving the classification results is to improve the ability of $f$ to distinguish different classes. Dirichlet Energy, as a powerful operator, reflects the smoothing of a given GNN function $f$. 
Contrary to smoothing nodes in the same class, we would like to make the nodes of different classes more distinguishable.
Therefore, we make use of Dirichlet Energy at a high level. Rather than consider edges across class directly, we regard the separated sub-graphs as nodes. All nodes in the same class are merged into a center node, ignoring the intra-class edges and preserving inter-class edges:
\begin{equation}
    \label{eq:interA}
    \bar{A} =   P^T A P  ,
\end{equation}
where $\bar{A}$ is the adjacency matrix of the merged graph. The node representations of merged graph is $\bar{H}^{(L)} =  P^T H^{(L)} $. 

Based on $\bar{A}$  and $\bar{H}^{(L)}$, we have the inter-class Dirichlet Energy:
\begin{equation}
    \label{eq:inter}
    E_{inter}  \left( H^{(L)} \right)
    = \frac{1}{2} \sum_{i,j=1}^{C}  \bar{A}_{i j}
    \left\|\frac{{ \bar{H}^{ (L) }}_i}{\sqrt{ \bar{d}_{i} } }-\frac{ { \bar{H}^{(L) } }_j}{ \sqrt{ \bar{d}_{j} } }\right\|_{2}^{2}
    ,
\end{equation}
where $\bar{d}_i$ is the degree of node $i$ in the merged graph. The Inter-Energy Reg is defined in the form of margin loss:
\begin{equation}
    \mathcal{L}_{ inter } = \max \left(0, m - E_{inter}  \right), 
\end{equation}
where $m$ is a hyper-parameter of the margin loss. The parameter $m$ acts as a boundary of inter-class Dirichlet Energy. Ideally, when no inter-class edges exist, we wish the $E_{inter}$ to be as large as possible. However, even if we filter the edges with \autoref{eq:interA}, inter-class edges still exist and maximizing  $E_{inter}$ without any control would make nodes being classified mistakenly by the function $f$.

\paragraph{Layer-wise LEReg} 
Unlike other regularization terms, which only perform on the last layer as they perform on the predicted labels directly, our LEReg can be further extended to perform on each layer of GNN. An L-layer GNN model can be formulated as
\begin{equation}
    \label{eq:l layer gcn}
    H^{(l+1)}=\sigma \left( \Tilde{A} H^{(l)} W^{(l)}\right).
\end{equation}
By using LEReg on each layer as:
\begin{equation}
\begin{aligned}
& \mathcal{L}_{intra,l} = E_{intra} \left( H^{(l)} \right) ,\\
& \mathcal{L}_{inter,l} = \max \left(0, m - E_{inter,l} \left( H^{(l)} \right)   \right),
\end{aligned}
\end{equation}
where $H^{(l)}$ denotes hidden embeddings of the $l$-th layer in $f$. 
Given the set of training nodes $S_{train}$, the GNN model $f$ can be trained with the combined loss function
\begin{equation}
\label{eq:combined loss}
\begin{aligned}
    \mathcal{L} 
    = &\mathcal{L}_{sup} 
    + \sum_{l=1}^L \alpha_l \mathcal{L}_{intra,l} 
    + \sum_{l=1}^L \beta_l \mathcal{L}_{inter,l} \\
    = & - \frac{1}{ | S_{train} | }  
    \sum_{i \in S_{train}} 
    \sum_{j=1, \cdots, C} Y_{ij} \log \left( P_{ij} \right) \\
    & + \sum_{l=1}^L \alpha_l \mathcal{L}_{intra,l} 
    + \sum_{l=1}^L \beta_l \mathcal{L}_{inter,l}
\end{aligned}
\end{equation}
where $Y_{ij}=1$ if the ground truth label of node $v_i$ is $j$ and $Y_{ij}=0$ otherwise.
$\alpha_l$ and $\beta_l$ are the regularization factors for $\mathcal{L}_{intra,l}$ and $\mathcal{L}_{intra,l}$, respectively.  
To apply LEReg only on the final layer of $f$, we could 
sett $\alpha_l = 0$ and $\beta_l = 0$ for $l=1, \cdots, L-1$ in \autoref{eq:combined loss}.
By introducing LEReg into previous layers, we can preserve useful information from low levels to high levels. Specifically, on the one hand, layer-wise Intra-Energy Reg helps to preserve common patterns in node features and node embeddings of hidden units. These common patterns contain information of different levels of granularity.
On the other hand, the layer-wise Inter-Energy Reg empowers the GNN to distinguish the labels of different nodes in a hierarchical manner. 

\paragraph{Complexity Analysis} For a GNN model $f$ with $L$ layers and the size of the $i$-th hidden layer $F_l$ for $l=1, \cdots, L$ ($F_L = C$), the computational complexity of LEReg on the final layer is  $\mathcal{O} \left(|\mathcal{E}| C \right)$, where $|\mathcal{E}|$ is the number of edges in the graph.
Computing $\mathcal{L}_{ inter, L}$ in \autoref{eq:intra} costs $\mathcal{O} \left(|\mathcal{E}| C \right)$ for an efficient GPU-based implementation using sparse-dense matrix multiplications. 
Computing the probability $Q$ ( defined in \autoref{eq:q probability} ) originally costs $\mathcal{O} \left(|\mathcal{V}|^2 C \right)$. However, it could be reduced to $\mathcal{O} \left(|\mathcal{E}| C \right)$ as we only need $Q_{ij}$ for $(v_i, v_j) \in \mathcal{E}$. 
Computing $\bar{A}$ costs $\mathcal{O} \left(|\mathcal{E}| C \right)$, which is similar to computing \autoref{eq:intra}. For $E_{inter,L}$ in \autoref{eq:inter}, it costs $\mathcal{O} \left( C^3 \right)$. Therefore, the overall computational complexity is $\mathcal{O} \left( 3 |\mathcal{E}| C + C^3 \right)$ = $\mathcal{O} \left(  |\mathcal{E}| C \right)$, as $C \ll |\mathcal{E}|$. 
The overall computational complexity for all layers is $\mathcal{O} \left(  |\mathcal{E}| \left( \sum_{l=1}^L F_l \right)  \right)$.
\section{How LEReg boosts GNNs?}
\label{sec:theoretical analysis}

\subsection{Benefits of LEReg from local regularization}

\label{subsec: discusion Benefits of LEReg from local regularization}

Given a graph  $G=(A, X)$ and a set of nodes with ground truth labels $V_{train}$, the GNNs model $f$  learns to map from the node features and the local topology to the node labels. 
Usually, for the classification task, the more balanced the data for training is, the better the result of the classification task can get. 
Unfortunately, it is challenging to balance the data distribution for graph-structured data with methods like resampling, because not only node features but also graph structure play an important role in graph representation learning. Learning structure distribution is rather difficult, which is still under exploration.

The graph Laplacian regularization term is usually used in semi-supervised representation learning to provide graph structure information for a model $g(X)$. Recently, with the development of GNNs, more and more methods have been proposed to encode the structure information $A$ directly. 
However, different structure distributions exist in a graph, and they help the GNNs model $f$  to distinguish different nodes. 
Using a global graph Laplacian regularization term or encoding $A$ with a uniform GNNs model $f(A,X)$ may constrain the model to further capture the specific information for each class. 

\paragraph{Theoretical analysis for distributions of different sub-graphs}

As we show in \autoref{fig:case_study}, there is large difference among the speed of node embedding smoothing for different classes. Here we give a more theoretical discussion. The convergence rate of node embeddings reflects the speed of smoothing over the graph. \cite{defsubspace} analyses the connection between the structure and the smoothing degree after one layer of Graph Convolution (GC). 

\begin{theorem}
\label{theorem:oono convergence}
For any initial value $H^{(0)}$ for node features, the output of l-th layer $H^{(l)}$ satisfies $d_{\mathcal{M}}\left(H^{(l)}\right) \leq$ $(s \lambda)^{l} d_{\mathcal{M}} \left(H^{(0)}\right) .$ 
In particular, $d_{\mathcal{M}}\left(H^{(l)}\right)$ exponentially converges to 0 when $s \lambda<1 .$
\end{theorem}

Here, $s_l$ is the singular value of $W^{(l)}$ in \autoref{eq:l layer gcn}, 
$\mathcal{M}:=U \otimes \mathbb{R}^{C}=\left\{\sum_{m=1}^{M} e_{m} \otimes w_{m} \mid w_{m} \in \mathbb{R}^{C}\right\}$ and $\left(e_{m}\right)_{m \in[M]}$ is the orthonormal basis of $U$. 
We denote $U$ by the eigenspace associated with the largest eigenvalue $\lambda_{N}$ of $D^{-\frac{1}{2}} A D^{-\frac{1}{2}}$.
The distance between $H$ and $\mathcal{M}$ is denoted as $d_{ \mathcal{M} } ( H ) := \inf \left\{ \| H-Y \|_{\mathrm{F}} \mid Y \in \mathcal{M}  \right\}$.
Then, we have $d_{\mathcal{M}}\left(H^{(l+1)}\right) \leq s_{l} \lambda d_{\mathcal{M}}\left(H^{(l)}\right) .$ 
Following the above theorem on the convergence rate on $H^{(l)}$, we discuss the relationship between the convergence rate of sub-graphs and the whole graph, as shown in the next theorem. 
The structure of sub-graph $S$ is unknown, which makes it hard to accurately analyze its eigenvalues, and the eigenvalues of sub-graphs and the whole graph do not have a fixed numeric relation. So we loosen the problem to discuss the eigenvalue gap: 
\begin{theorem}
For $\lambda$, the second largest eigenvalue of $D^{ - \frac{1}{2} } A D^{-\frac{1}{2}} $, we denote its upper bound as $\lambda_u$, and its lower bound as $\lambda_l$. $G$ denotes a graph and $S$ is a sub-graph of $G$. $\lambda_S$ is the second largest eigenvalue of normalized Laplacian matrix of sub-graph $S$. Then we have:
\begin{itemize}
    \item  $\lambda_{uS}\leq \lambda_{uG}$ and $\lambda_{lS}\leq \lambda_{lG}$
    \item For two non-isomorphic sub-graphs of $G$, $S_1$ and $S_2$ ,  $\lambda_{S_1} \not = \lambda_{S_2}$.
\end{itemize}
\end{theorem}
\begin{proof}

We adopt the relation between the eigenvalues gap and the conductance of graph to give the proof.
The conductance of $G$ is defined as $\phi(G)= \min_S \phi(S)$, where
\begin{gather*}
    \phi(S)=\frac{\sum_{s\in S,j \in \bar{S}} A_{ij}}{min(d(S),d(\bar{S}))}, \
    \bar{S} = G \backslash S, \ 
    d(S)=\sum_{i\in S}d_i
\end{gather*}


The Cheeger inequality \cite{cheeger} relates the eigenvalue gap and conductance. It states that for a connected graph $G$, we have $ 2 \phi(G) \geq \lambda^* \geq \frac{\phi^2(G)}{2}$, where $\lambda^*$ is the smallest non-zero eigenvalue of $L_{sym}$. Then we have $\lambda =1-\lambda^*$, $\lambda_l = 1-2\phi(G)$ and $\lambda_u = 1-\frac{\phi(G)^2}{2}$. 

For any $S$, $\phi(S)\geq \phi(G)$, which leads to $\lambda_{uS}\leq \lambda_{uG}$ and $\lambda_{lS}\leq \lambda_{lG}$. 
Similar to the above proof, different subsets $S_1$ and $S_2$ may have different conductance, thus different bounds of $\lambda$.
\end{proof}

This may not lead to a tight bound. However, it could still provide some insights. For a sub-graph $S$, more edges across $S$ and $\bar{S}$ makes larger $\phi(S)$. Larger $\phi(S)$ implies larger $\phi(G)$ to some extend, which leads to smaller $\lambda$ of $G$, as the smallest non-zero eigenvalue of $L_{sym}$ is strongly related to the structure of $G$. Given Theorem \ref{theorem:oono convergence}, if the sub-graphs in $G$ share lots of edges with the other parts, the $H^{(l)}$ would converge fast and if they share few edges, the $H^{(l)}$ would converge slowly. 


Based on the analysis of eigenvalues above, we can conclude the local smoothing phenomenon as:

\begin{lemma}[Convergence rate of local graphs]
(1) Local embedding spaces converge faster than or equal to the global graph. 
(2) If a sub-graph shares more edges with the rest of the graph, it would converge faster
\end{lemma}

Therefore, using global regularization terms may hinder the ability of GNNs to capture distinct patterns of different class nodes. LEReg separates the graph into sub-graphs with Intra-Energy Reg and reduces the conductance with Inter-Energy Reg implicitly.

\subsection{Learning graph structure and spreading labels}
\label{subsec: discusion Extra information from the supervision of LEReg}

\paragraph{Graph structure affects the learning of GNNs.} The success of the existing GNNs relies on one fundamental assumption, i.e., the original graph structure is reliable. However, this assumption is usually unrealistic, since the graph in reality is inevitably noisy or incomplete. 
The more noise a graph contains, the worse node representations GNN models learn. 
Graph regularization, as an unsupervised term applied on a GNNs model’s representations during training, 
can provide extra supervision information for the model to learn a better graph structure in a graph.
For example, Wen et al. \cite{GSLforRobust} proposed to learn the graph structure by preserving low-rank, sparsity and the feature smoothing using three regularization terms.

\paragraph{Limitations of global regularization term}
Regularization terms, as an efficient method to boost Neural Networks, have been well-studied. For graph-structured data, the most common graph Laplacian regularization $\mathcal{L}_{\text {laplacian }}=\sum_{(i, j) \in \mathcal{E}}\left\| H_{i}^{ (L) }-Z_{j}^{ (L) }\right\|_{2}^{2}$ is used to capture the structure information. However, the assumption it bases on is that the graph structure is reliable. The improvement of applying global graph Laplacian regularization is limited as the structure information it captures is usually incomplete and noisy. 
Besides, the original graph convolution operator $g_{\theta}$ is approximated by the truncated Chebyshev polynomials as $g_{\theta} * Z=\sum_{k=0}^{K} \theta_{k} T_{k}(\tilde{\Delta}) Z$ \cite{hammond2011wavelets}, where $K$ means to keep $K$-th order of the Chebyshev polynomials, and $\theta_{k}$ is the parameters of $g_{\theta}$ as well as the Chebyshev coefficients. 
$T_{k}(x)=2 x T_{k-1}(x)-T_{k-2}(x)$ is the Chebyshev polynomials, with $T_{0}(x)=1$ and $T_{1}(x)=x$.  $K$ is equal to 1 for today's spatial GNNs such as GCN \cite{kipf-gcn}, GAT \cite{gat} and so on. Thus, the global graph Laplacian regularization does not learn information more than the GNNs.

\paragraph{Joint learning of graph structure and node embeddings with the supervision of LEReg}

Similar to the other global graph regularization terms, LEReg could provide extra supervision information for the nodes. The advantage of our method is that we combine the learning of node embeddings and graph structure. Graph Laplacian regularization  \cite{zhou2004learning,ando2007learning}  helps to learn the node embeddings given the original graph. Han et al. \cite{Preg-yang2020rethinking} propose 
$\mathcal{L}_{P-r e g}=\frac{1}{N} \phi(Z, \tilde{A} Z)$
to improve the node embeddings by one more aggregation of the final representations, where $\phi$ is a function that measures the difference between $Z$ and $\tilde{A} Z$, which is also based on the original graph topology structure. 
In our method, the two regularization terms are performed on the learnt graph structure. We perform Intra-Energy Reg with $\hat{A} = Q \odot A $ and $Q$ is defined in \autoref{eq:q probability}, while Inter-Energy Reg with $\bar{A} = P^T A P $.  
$\hat{A}$ and $ \bar{A}$ serve as the new graph structure for learning node embeddings. 
By reconstructing the graph, the graph structure is cleaner and the node embeddings learned are more robust.

On the other hand, we separate the graph using the predicted class posterior probability $P$. We bridge the gap between labeled nodes and unlabeled nodes, making the label information spread on the graph for intra-classes.  
This is similar to the label propagation that spread label information through the edges of a graph. 
Furthermore, Inter-Energy Reg is class-centric, making each class distinguishable by the GNNs. Minimizing the Inter-Energy Reg equals to doing node clustering with the supervision of labels.

In conclusion, with LEReg we learn the graph structure and node embedding simultaneously. Introducing additional class information to GNNs, LEReg further provides a new class-centric view to boost the results of node representation learning.


\subsection{Connection to Deep GNNs}

Recent studies begin to explore GNNs with more layers, wishing to capture information from further nodes. Ming and Chen et al. \cite{gcnii} prove that for a self-looped  graph  $\tilde{G}$ and a graph signal $\mathrm{x}$. 
A $K$-layer GCN with residual connection and identity mapping can express a $\mathrm{K}$ order polynomial filter  $\left( \sum_{ \ell=0 }^{K} \theta_{\ell} \tilde{\mathbf{L}}^{\ell}\right) \mathbf{x}$ with arbitrary coefficients $\theta$. Similar methods \cite{resgcn, jknet, StrongerGCN} adopt the same idea to combine residual connection and identity mapping with Graph Convolution. 

\paragraph{The limitation of deep GNNs}
Deeper is not always better. Deep GNNs suffer from the problem of over-smoothing \cite{defsubspace, deeper_insight} and over-fitting \cite{grand,dropedge}. For the over-smoothing issue, recalling Theorem \ref{theorem:oono convergence}, GNNs do not improve (or sometimes worsen) their predictive performance as we pile up many layers and add non-linearity. 
Its output exponentially approaches the set of signals that carry information of the connected components of the graph and node degrees only for distinguishing nodes. 
On the other hand, over-fitting, as well as gradient vanishing problems generally appear for all deep neural networks. 
Even with residual connections, the training becomes harder with deeper GNNs. 
The truth is, we do not need such deep layers. For a connected graph $G$, the diameter of which is $d_G$ (the length of the longest path between two nodes), a GNN model with $d_G$ is enough to view all the other nodes in the graph.
 
\paragraph{LEReg adaptively capture the important information with Intra-Energy Reg and Inter-Energy Reg} 
We first prove that minimizing the Intra-Energy Reg for the final outputs reaches the same output as applying infinite graph convolution on the sub-graphs when we get $\hat{A}$ with \autoref{eq:discreate mask}. 

\begin{lemma}
\label{lemma:gradient}
The aggregation step in GCN is equivalent to running gradient descent for $\mathcal{L}_{ \text{intra} }$ with a step size of one.
\end{lemma}

\begin{proof}
\begin{equation}
\begin{aligned}
\frac{  \partial \mathcal{L}_{intra,L}  } {\partial H^{(L)}_i } 
&= \sum_{ \left(v_{i}, v_{j}\right) \in \mathcal{E} } 
\hat{A}_{i j} \left( H_{i}^{(L)} -  H_{j}^{(L)} \right) \\
H^{(L)}_i - \frac{  \partial \mathcal{L}_{intra,L}  } {\partial H^{(L)}_i } 
&=\sum_{ \left(v_{i}, v_{j}\right) \in \mathcal{E} } \hat{A}_{i j} H_{j}^{(L)}
\end{aligned}
\end{equation}
\end{proof}

Lemma \autoref{lemma:gradient} implies  that if  we minimize $\mathcal{L}_{intra}$ iteratively, the outputs $H^{ (L) }$ in one sub-graph converge to the same point. Under the definition of $\hat{A}$, the edges across classes are filtered. 
Consequently, each sub-graph converges to a specific point, which is equivalent to applying infinite graph convolution on each sub-graph \cite{deeper_insight}.


The results of applying deep GNNs are not desirable, but only being able to choose the number of layers $L$ from a discrete set $\mathbb{N}_{+}$ is also not ideal. 
A balance between the size of the receptive field (the farthest nodes that a node can obtain information from) and the preservation of discriminating information is critical for generating useful node representation \cite{Preg-yang2020rethinking}. 
Thus, the regularization factor $\alpha_L \in \mathbb{R}$ in \autoref{eq:combined loss} is vital in enforcing a flexible strength of regularization on the GNN. 

What's more, using $\mathcal{L}_{intra,L}$ with $\hat{A}$ further alleviates the over-smoothing by constraining the message passed through edges. 
Consequently, GNNs with $\mathcal{L}_{intra,L}$ are able to explore a larger size of receptive field with a wide range value of $\alpha_L$.






\section{Experiment}
\label{sec:experiment}

\begin{table}[pth]
\centering
\caption{Dataset statistics.}
\begin{tabular}{lccccc}
\toprule
            & \multicolumn{1}{c}{ $|\mathcal{V}|$ } & \multicolumn{1}{c}{$|\mathcal{E}|$} & \multicolumn{1}{c}{F} & \multicolumn{1}{c}{C} & \multicolumn{1}{c}{Avg. Degree~} \\
          \midrule 
Cora~      & 2,708  & 5,278  & 1,433     & 7        & 1.95                    \\
Citeseer~  & 3,327  & 4,552  & 3,703     & 6        & 1.37                 \\
Pubmed~    & 19,717 & 44,324 & 500       & 3        & 2.25                \\
Computers~ & 13,752 & 245,861 & 767      & 10       & 17.88             \\
Photo~ & 7,560 & 119,081 & 745      & 8       & 15.57             \\
Wisconsin~ & 251 & 257 & 1703      & 5       & 1.03            \\

\bottomrule
\end{tabular}
\label{table:statistics}
\end{table}

We conduct experiments on six datasets \cite{yang2016revisiting, Pitfalls, Geom-GCN} with different scales of node classification benchmark to evaluate the performance of LEReg against state-of-art Graph Neural Network baselines. In addition to experiments on graph datasets of random splits, we also conduct experiments on the standard split of 3 graph datasets.

In this section, we aim to answer three questions: 
\textbf{Q1}: How LEReg helps to boost GNNs compared to other regularization terms? 
\textbf{Q2}: Can Energy Regularizations help to reach new state-of-the-art (SOTA) performances of node classification tasks?
\textbf{Q3}: How do the Inter-Energy Reg and Intra-Energy Reg help GNNs with node classification respectively? 
\textbf{Q4}: How Intra-Energy Reg and Intre-Energy Reg help the training of GNNs and what is the relation between Dirichlet Energy and accuracy.

\subsection{Experiment Setup}
\label{subsec:exp_setup}

\paragraph{Datasets} 
We focus on six widely used graph structured datasets:
(1) In Cora, Citeseer and Pubmed, nodes represent papers and edges between nodes represent the citation relationships. Given abstracts as bag-of-words node features, the task is to map papers to their respective research category. 
(2) In Computers and Photo, nodes represent goods and edges represent that two goods are frequently bought together. Given product reviews as bag-of-words node features, the task is to map goods to their respective product category.
(3)
In Wisconsin, nodes represent web pages and edges represent hyperlinks between them. Node features are the bag-of-words representation of web pages. The task is to classify the nodes into one of the five categories, student, project, course, staff, and faculty.
The statistic of the datasets is summarised as \autoref{table:statistics}.

\paragraph{Settings and implementation}
We evaluated the node classification accuracy for 3 backbone models on six popular datasets using random splits to compare with other graph regularization terms. The train/validation/test split of all the 6 datasets are 20 nodes/30 nodes/all the remaining nodes per class, as recommended by \cite{Pitfalls}.
We conducted each experiment on 5 random splits and 5 different trials for each random split.

We use the Adam Stochastic Gradient Descent (SGD) optimizer \cite{Adam_sgd} with a learning rate of 0.01 and early stopping with a patience of 100 epochs to train all the models. The implementation of LEReg is based on PyTorch \cite{paszke2019pytorch}. 
To enable more robust comparisons, we perform grid hyper-parameter search for each model with Neural Network Intelligence (NNI) \footnote{https://nni.readthedocs.io/en/stable/contents.html}, and report the test accuracy based on the best accuracy on the validation set. 
The search space for $\alpha_l$ and $\beta_l$ for $l=1, \cdots, L$ is 10 values evenly chosen from $[0,1]$ based on the validation accuracy. If not specified, LEReg is added on each layer of the backbone models.

For the standard split experiments using the state-of-the-art methods, the code we used is either the official code released by the authors or from the PyTorch-geometric \cite{torch_geometric} implementation.


\begin{table*}
\caption{
    Summary of node classification accuracy (\%) results on random splits. The highest results for each dataset with different backbones are bolded.
}
\label{table:random split}
\centering
    \begin{tabular}{lllllllll}
    \toprule
    &  & Cora & Citeseer & Pubmed & Computers & Photo & Wisconsin  \\ 
     \midrule \multirow{5}*{\textbf{GCN}}
     & Vanilla& 76.03$\pm$4.25 & 67.14$\pm$1.79 & 74.17$\pm$4.91 & 80.25$\pm$3.15 & 87.42$\pm$1.90 & 56.37$\pm$6.20 \\ 
     & Laplacian& 75.29$\pm$6.20 & 67.35$\pm$1.69 & 73.90$\pm$4.85 & 79.68$\pm$2.74 & 87.49$\pm$1.85 & 56.24$\pm$6.19 \\ 
     & Label Smoothing& 75.25$\pm$4.12 & 66.45$\pm$2.39 & 73.89$\pm$4.39 & 80.84$\pm$2.33 & 86.41$\pm$2.17 & 55.54$\pm$6.56 \\ 
     & P-reg& 76.96$\pm$3.57 & 67.23$\pm$1.53 & 74.37$\pm$4.80 & 80.61$\pm$3.01 & 88.35$\pm$1.82 & 56.49$\pm$6.34 \\ 
     & LEREG& \textbf{77.71}$\pm$2.92 & \textbf{68.26}$\pm$2.69 & \textbf{75.06}$\pm$4.53 & \textbf{81.82}$\pm$2.98 & \textbf{89.33}$\pm$1.74 & \textbf{58.17}$\pm$4.90 \\ 
     \midrule \multirow{5}*{\textbf{GAT}}
     & Vanilla& 77.64$\pm$1.57 & 64.93$\pm$1.47 & 74.24$\pm$2.54 & 80.17$\pm$2.66 & 86.29$\pm$2.15 & 53.77$\pm$7.30 \\ 
     & Laplacian& 77.52$\pm$1.76 & 65.09$\pm$1.66 & 74.82$\pm$1.87 & 80.25$\pm$2.72 & 86.31$\pm$2.43 & 53.65$\pm$7.59 \\ 
     & Label Smoothing& 76.99$\pm$1.80 & 64.84$\pm$1.20 & 73.30$\pm$3.78 & 77.94$\pm$2.94 & 87.37$\pm$1.88 & 53.62$\pm$8.18 \\ 
     & P-reg& 77.95$\pm$1.39 & 65.00$\pm$1.37 & 74.60$\pm$3.46 & 80.85$\pm$2.30 & \textbf{87.57}$\pm$1.90 & 54.08$\pm$7.03 \\ 
     & LEREG& \textbf{80.92}$\pm$0.86 & \textbf{66.06}$\pm$5.50 & \textbf{77.27}$\pm$2.44 & \textbf{81.31}$\pm$3.01 & 87.30$\pm$1.89 & \textbf{58.50}$\pm$5.64 \\ 
     \midrule \multirow{5}*{\textbf{SGC}}
     & Vanilla& 78.64$\pm$1.58 & 67.53$\pm$2.61 & 75.36$\pm$1.92 & 82.00$\pm$1.98 & 87.84$\pm$2.54 & 51.30$\pm$5.65 \\ 
     & Laplacian& 78.73$\pm$1.86 & 67.81$\pm$2.60 & 75.61$\pm$2.03 & 82.33$\pm$1.23 & 88.02$\pm$2.76 & 49.62$\pm$6.38 \\ 
     & Label Smoothing& 76.32$\pm$1.81 & 66.80$\pm$1.98 & 75.49$\pm$2.02 & 80.88$\pm$2.41 & 87.60$\pm$1.65 & 51.42$\pm$6.00 \\ 
     & P-reg& 78.55$\pm$1.10 & 67.95$\pm$1.58 & 75.55$\pm$2.15 & 82.20$\pm$1.31 & 88.37$\pm$2.04 & 51.02$\pm$5.97 \\ 
     & LEREG& \textbf{80.29}$\pm$2.62 & \textbf{68.77}$\pm$3.71 & \textbf{77.60}$\pm$1.31 & \textbf{83.14}$\pm$1.37 & \textbf{89.14}$\pm$2.23 & \textbf{52.31}$\pm$4.33 \\ 
        \bottomrule
    \end{tabular}
\end{table*}

\subsection{Results}

\paragraph{Boosting GNNs with LEReg}
We applied LEReg on GCN \cite{kipf-gcn}, GAT \cite{gat} and SGC \cite{sgc}, respectively.
GCN is a typical convolution-base model via an approximation of spectral graph convolutions. GAT introduces the attention mechanism into GNN and achieves promising results. 
These two models are representative of a broad range of GNNs. 
On the other hand, SGC reduces the excess complexity of GCNs through successively removing nonlinearities and collapsing weight matrices between consecutive layers. It provides a simplified graph convolution framework for large-scale graph datasets with high speed. 
In addition, we also applied global Laplacian regularization, 
label smoothing, 
and P-reg to the models, respectively, in order to give a comparative analysis on the effectiveness of LEReg. 
Laplacian regularization is a typical  regularization term in semi-supervised representation learning to provide graph structure information, defined as $\mathcal{L}_{\text {laplacian }}=\sum_{(i, j) \in \mathcal{E}}\left\| H_{ i }^{ (L)  }- H_{ j }^{ (L)  }. \right\|_{2}^{2}$.
Label smoothing \cite{szegedy2016rethinking,muller2019does} is a general method to improve the generalization capability of a model and has been adopted in many state-of-the-art deep learning models \cite{GPipe-nips-HuangCBFCCLNLWC19,real2019regularized}. It softens the one-hot hard targets $y_{c}=1, y_{i}=0 \, \forall i \neq c$ into $y_{i}^{L S} = (1-\alpha) y_{i}+ \alpha / C$, where $c$ is the correct label and $C$ is the number of classes. 
P-reg \cite{Preg-yang2020rethinking} is proposed as a variant of global graph Laplacian regularization to improve the GNNs with label information. It is defined as $\mathcal{L}_{P-r e g}=\frac{1}{N} \phi(Z, \tilde{A} Z)$, where $\phi$ is a function that measures the difference between $Z$ and $\tilde{A} Z$. 

The mean value and standard deviation are reported in \autoref{table:random split}. The regularization factor $\mu$ for global graph Laplacian regularization, label smoothing and P-reg is determined by grid search with NNI using the validation accuracy. The search space is $0.1, 0.2, \cdots, 1$. The $\mu$ is the same for each cell (model × dataset). 

\autoref{table:random split} shows three LEReg significantly improves three basic GNN models on node classification tasks. In general, graph Laplacian and label smoothing show limited improvements on GCN, GAT and SGC. For graph Laplacian regularization, it does not provide more information than graph convolution ( see \autoref{subsec: discusion Extra information from the supervision of LEReg} ), so the best result for it is to maintain the original accuracy of the backbone model. 
The improvements brought by LEReg are also consistent in all the cases except for GAT on the Photo dataset. Although P-reg improves GAT on Photo dataset, LEReg shows competitive results in other cases. This demonstrates  LEreg can be well applied to other GNN methods. Here we answer question \textbf{Q1}.

\paragraph{Comparing with SOTAs}
To answer question \textbf{Q2}, we compare GCN + LEReg, GAT + LEReg and SGC + LEReg with State-Of-The-Art methods. 
APPNP \cite{appnp}
is the newly proposed SOTAs. IncepGCN \cite{deeper_insight} train overcomes the limits of the GCN model with shallow architectures with co-training and self-training approaches.
GraphMix \cite{graphmix} adopts the idea of co-training \cite{blum1998combining} to use a parameters-shared fully-connected network to make a GNN more generalizable. It combines many other semi-supervised techniques such as Mixup \cite{zhang2017mixup}, entropy minimization with Sharpening \cite{EntropyMinimization}, Exponential Moving Average of predictions \cite{Weight-averaged} and so on. 
ResGCN \cite{resgcn} and JKNet \cite{jknet} are two strong benchmarks with residual connection and jumping layers respectively.
GroupNorm normalizes nodes within the same group independently to increase their smoothness and separates node distributions among different groups with differentiable group normalization (DGN).
DropEdge \cite{dropedge} randomly removes a certain number of edges from the input graph at each training epoch, acting as a data augmenter and also a message-passing reducer. By doing so, it alleviates the over-fitting and over-smoothing issues.
Note that DropEdge in \autoref{table:summary} is not a specific method, as DropEdge reports in \cite{dropedge}, we take the results of the best backbone for each dataset, that is, IncepGCN (with DropEdge) for Cora, IncepGCN (with DropEdge) for Citeseer, GCN (with DropEdge) for Pubmed. 

\autoref{table:summary} reports the test accuracy of node classification on three citation datasets with the public split.  
We report the mean of the accuracy of 10 different trials. The search space for $\alpha_l$ and $\beta_l$ for $l=1, \cdots, L$ is $0.1, \cdots, 1.0$. 
From \autoref{table:summary}, we have the following observations: (1)  GNNs with LEReg outperform the SOTAs on Cora and Citeseer. (2)
Although GNNs with LERegs do not match the SOTAs, they still show significant improvements compared to the original methods.





\begin{table}[pht]
\caption{
    Summary of node classification accuracy (\%) results on Cora, Citeseer, and Pubmed. The number in parentheses corresponds to the depths of the model. The highest result for each dataset is bolded.
}
\label{table:summary}
\centering
\begin{tabular}{l|lll}
\toprule
\textbf{Method}  & \textbf{ Cora}& \textbf{Citeseer} & \textbf{Pubmed}  \\ 
\midrule
IncepGCN  & 81.7   & 70.2     & 77.9      \\ 
APPNP     & 83.9    & 72.2    & 80.4     \\ 
GroupNorm & 81.1  & 69.5 & 79.5 \\ 
DropEdge  & 83.5 & 72.7 & 79.6 \\ 
ResGCN    & 78.8     & 70.5  & 78.6     \\ 
JKNet     & 81.1    & 69.8    & 78.1      \\ 
GraphMix & 83.5  & 73.8 & \textbf{80.8} \\
GCN       & 81.1 & 70.8 & 79.0  \\ 
SGC       & 81.0 & 71.9 & 78.9  \\ 
GAT  & 83.2 & 71.3 & 78.0 \\
\midrule
GCN  & 83.5    & 73.6  & 78.9      \\ 
E-SGC  & 83.5 & \textbf{74.2}  & 79.1 \\
E-GAT  & \textbf{84.1} & 72.2    & 79.2      \\ 

\bottomrule
\end{tabular}
\end{table}

\subsection{Ablation Study}

We conduct an ablation study to examine the contributions of different components in LEReg, to answer the research question \textbf{Q3}.
Due to the space limit, we only provide the results on Cora.
Note that this subsection mainly focuses on analyzing the effectiveness of the components of LEReg without the concern of pushing state-of-the-art results.
So, we do not perform delicate hyper-parameter selection with NNI like what \autoref{subsec:exp_setup} introduces. 
We employ GCN as the backbone. The hidden dimension, learning rate and weight decay are fixed to 64, 0.01 and 0.0005, receptively.
We compare four variants as follows:

\begin{itemize}
    \item GCN: without Inter-Energy Reg and Intra-Energy Reg
    \item GCN+InterE: without Intra-Energy Reg
    \item GCN+IntraE: without Inter-Energy Reg
    \item GCN*: with both Intra-Energy Reg and Inter-Energy Reg
\end{itemize}

In \autoref{fig:ablation}, we show the results of these four variants with depths of 4, 16, and 64. The height of the bars in \autoref{fig:ablation} shows the accuracy of node classification with Cora, from which we have three observations. 
First, with Inter-Energy Reg and Intra-Energy Reg, GCN* outperforms the other variants in most cases, suggesting that each of the designed regularization terms contributes to the success of LEReg. 
To be specific, LEReg enables GCN to be more powerful with shallow architectures and alleviate over-smoothing with deep layers.
Second, GCN with Inter-Energy Reg reaches better results than GCN with Intra-Energy Reg with layer of 64 in all three datasets, suggesting that Inter-Energy Reg is more powerful in alleviating, which is consistent with intuition and anticipation.
Third, without Inter-Energy Reg, GCN outperforms the other variants when we take layer as 16, which demonstrates the significance of the proposed Intra-Energy Reg for semi-supervised learning.

\pgfplotsset{
axis background/.style={fill=gallery},
grid=both,
  xtick pos=left,
  ytick pos=left,
  tick style={
    major grid style={style=white,line width=1pt},
    minor grid style=bgc,
    draw=none
    },
  minor tick num=1,
  ymajorgrids,
	major grid style={draw=white},
	y axis line style={opacity=0},
	tickwidth=0pt,
}

\begin{figure}[hp]
\centering
    \begin{tikzpicture}[scale=1]
	    \begin{axis}[
	    width=\linewidth ,
	    height=0.45\linewidth,
        ylabel={
        ACC-Cora
        },
        xticklabels={layer=4, layer=16, layer=32},
        xtick={1,2,3},
        x tick label style={opacity=0},
        tick label style={font=\small},
        enlarge x limits=0.3,
        ybar=2pt, 
        /pgf/bar width=10pt,
        ymajorgrids,
        tickwidth=0pt,
        yticklabel style={
        /pgf/number format/fixed,
        /pgf/number format/precision=5
        },
        scaled y ticks=false,
        every axis title/.append style={at={(0.1,0.8)},font=\bfseries},
        legend style = {
        font=\scriptsize, 
          draw=none, 
          draw opacity=0,
          at={(0.5, 0.98)},
          anchor=south,
          fill=none,
          column sep = 1pt, 
          /tikz/every even column/.append style={column sep=5mm},
          legend columns = -1},
        ]

		\addplot[draw=none, fill=bb] coordinates {
          (1, 80.4)
          (2, 64.9)
          (3, 28.7)
        }; \addlegendentry{GCN}
        
        \addplot[draw=none, fill=gg] coordinates
        {
          (1, 82)
          (2, 67.4)
          (3, 39.6)
        }; \addlegendentry{GCN+InterE}
        
        \addplot[draw=none, fill=yy] coordinates
        {
          (1, 82.6)
          (2, 70.3)
          (3, 37.6) 
        }; \addlegendentry{GCN+IntraE}
        
        \addplot[draw=none, fill=rr] coordinates
        {
          (1, 82.9)
          (2, 70.5)
          (3, 45.2)
        }; \addlegendentry{GCN* }
        \end{axis}
    \end{tikzpicture}
    
    \begin{tikzpicture}[scale=1]
	    \begin{axis}[
	    width=\linewidth ,
	    height=0.45\linewidth,
        ylabel={
        ACC-Citeseer
        },
        xticklabels={layer=4, layer=16, layer=32},
        xtick={1,2,3},
        x tick label style={opacity=0},
        tick label style={font=\small},
        enlarge x limits=0.3,
        ybar=2pt,
        /pgf/bar width=10pt,
        ymajorgrids,
        tickwidth=0pt,
        yticklabel style={
        /pgf/number format/fixed,
        /pgf/number format/precision=5
        },
        scaled y ticks=false,
        every axis title/.append style={at={(0.1,0.8)},font=\bfseries},
        legend style = {
		  font=\small,
          draw=none, 
          draw opacity=0,
          at={(0.5, 0.98)},
          anchor=south,
          fill=none,
          column sep = 1pt, 
          /tikz/every even column/.append style={column sep=5mm},
          legend columns = -1},
	    ]
	    
		\addplot[draw=none, fill=bb] coordinates {
          (1, 67.6)
          (2, 18.3)
          (3, 20)
        }; 
        
        \addplot[draw=none, fill=gg] coordinates
        {
          (1, 68.3)
          (2, 53.9)
          (3, 30.9)
        }; 
        \addplot[draw=none, fill=yy] coordinates
        {
          (1, 67.9)
          (2, 37.7)
          (3, 29.3)
        }; 
        
        \addplot[draw=none, fill=rr] coordinates
        {
          (1, 69.1)
          (2, 48.5)
          (3, 30.6)
        }; 
        \end{axis}
    \end{tikzpicture}

    \begin{tikzpicture}[scale=1]
	    \begin{axis}[
	    width=\linewidth ,
	    height=0.45\linewidth,
        ylabel={
        ACC-Pubmed
        },
        xticklabels={layer=4, layer=16, layer=64},
        xtick={1,2,3},
        tick label style={font=\small},
        enlarge x limits=0.3,
        ybar=2pt,
        /pgf/bar width=10pt,
        ymajorgrids,
        tickwidth=0pt,
        yticklabel style={
        /pgf/number format/fixed,
        /pgf/number format/precision=5
        },
        scaled y ticks=false,
        every axis title/.append style={at={(0.1,0.8)},font=\bfseries},
        legend style = {
		  font=\small,
          draw=none, 
          draw opacity=0,
          at={(0.5, 0.98)},
          anchor=south,
          fill=none,
          column sep = 1pt, 
          /tikz/every even column/.append style={column sep=5mm},
          legend columns = -1},
	    ]
	    
		\addplot[draw=none, fill=bb] coordinates {
          (1, 76.5)
          (2, 40.9)
          (3, 35.3)
        }; 
        
        \addplot[draw=none, fill=gg] coordinates
        {
          (1, 78.9)
          (2, 66.5)
          (3, 53.5)
        }; 

        \addplot[draw=none, fill=yy] coordinates
        {
          (1, 77.9)
          (2, 73)
          (3, 45.3)
        }; 
        
        \addplot[draw=none, fill=rr] coordinates
        {
          (1, 78)
          (2, 69)
          (3, 50.9)
        }; 
        \end{axis}
    \end{tikzpicture}

    \caption{
    Ablation Study over three  citation datasets. Each bar represents an algorithm under a certain number of layers, plotted by number of layers on the horizontal axis and the accuracy on the vertical.
    }
    \label{fig:ablation}
\end{figure}
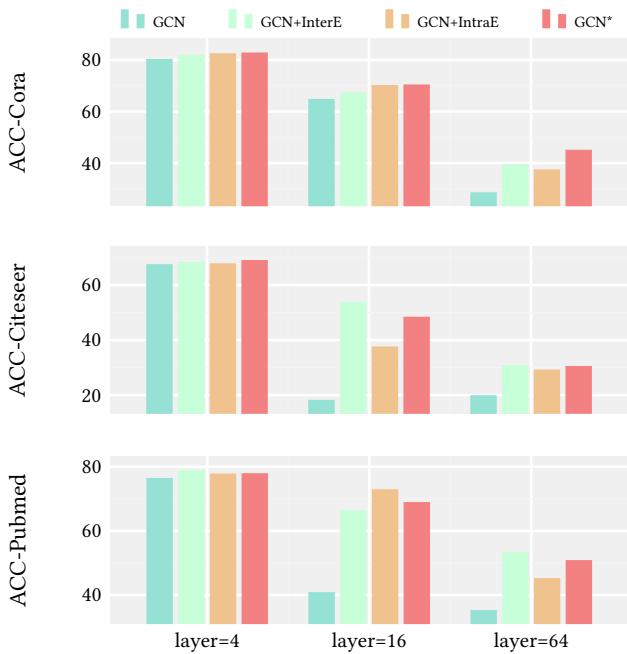

\begin{figure*}
    \centering
    \subfigure[The Intra Energy of node representations on Cora.]
    {
        \includegraphics[width=3in]{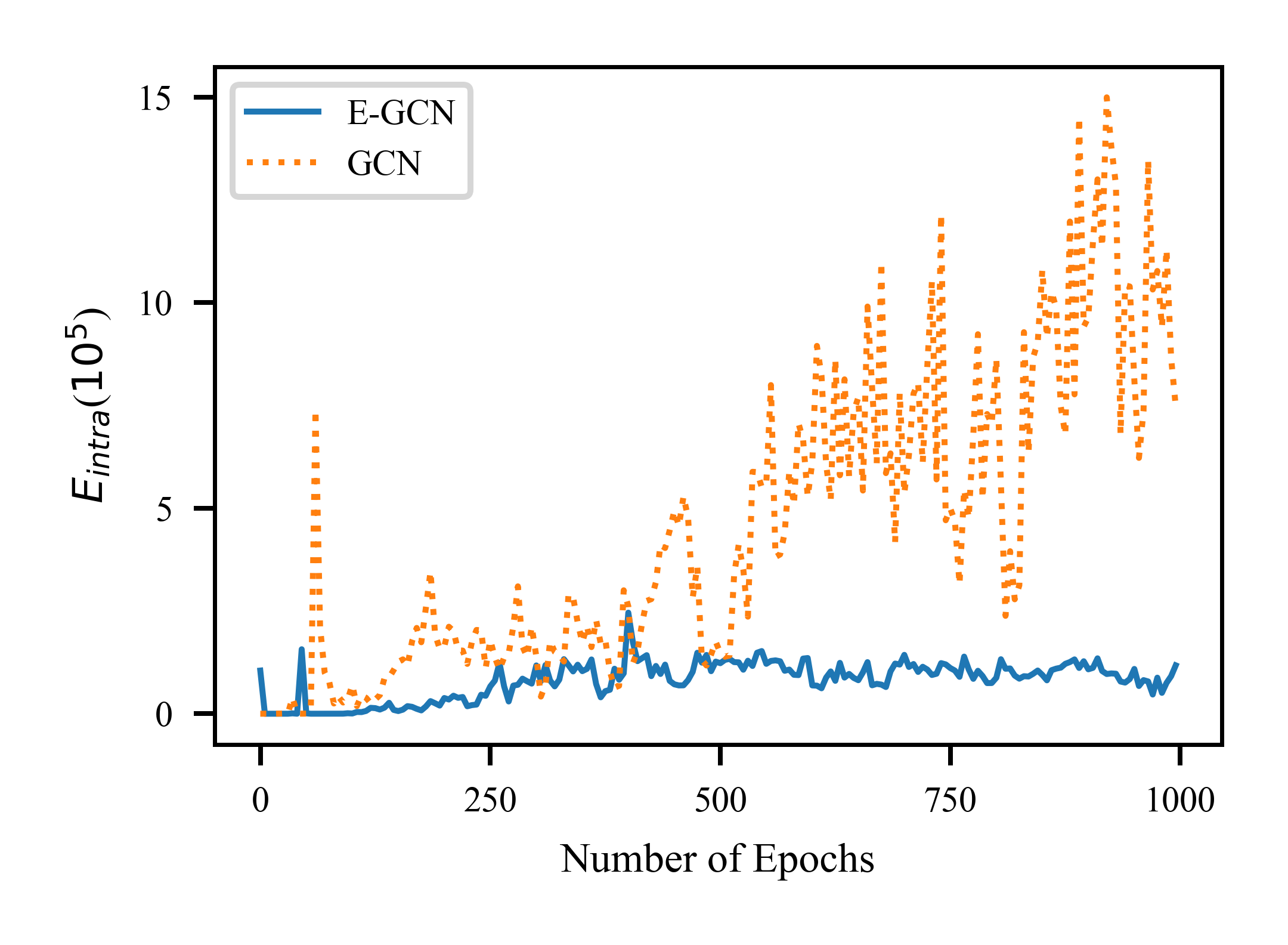}
    }
    \subfigure[The accuracy of node classification on Cora.]
    {
        \includegraphics[width=3in]{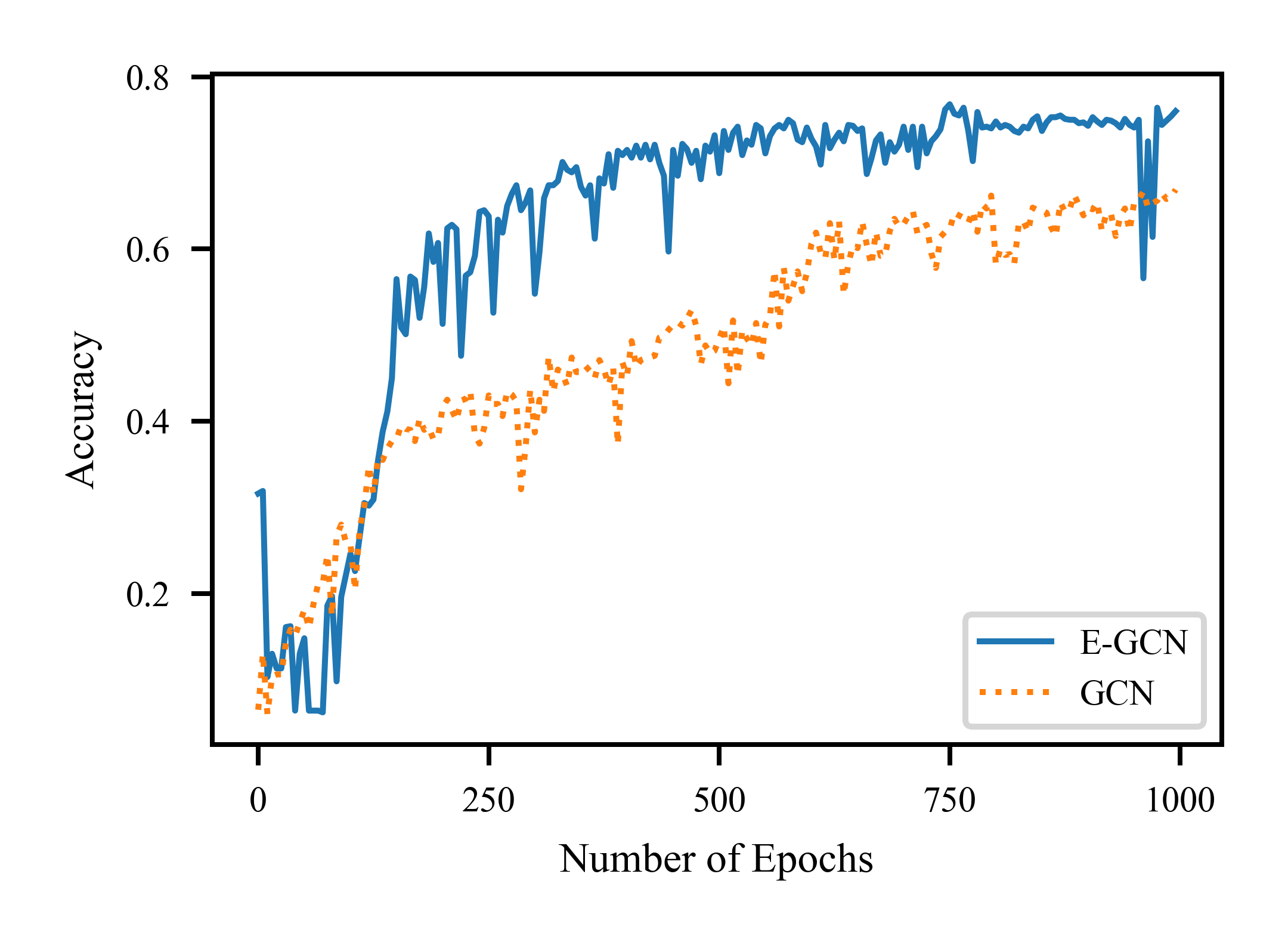}
    }

    \caption{
    The trends of accuracy and Intra-Energy during training. (a) The Intra-Energy of the embeddings of GCN and E-GCN. (b) The accuracy of node classification task on Cora with GCN and E-GCN (GCN with LEReg).
    }
    \label{fig:VS}
\end{figure*}


\begin{figure}
    \centering
    \includegraphics[width=3in]{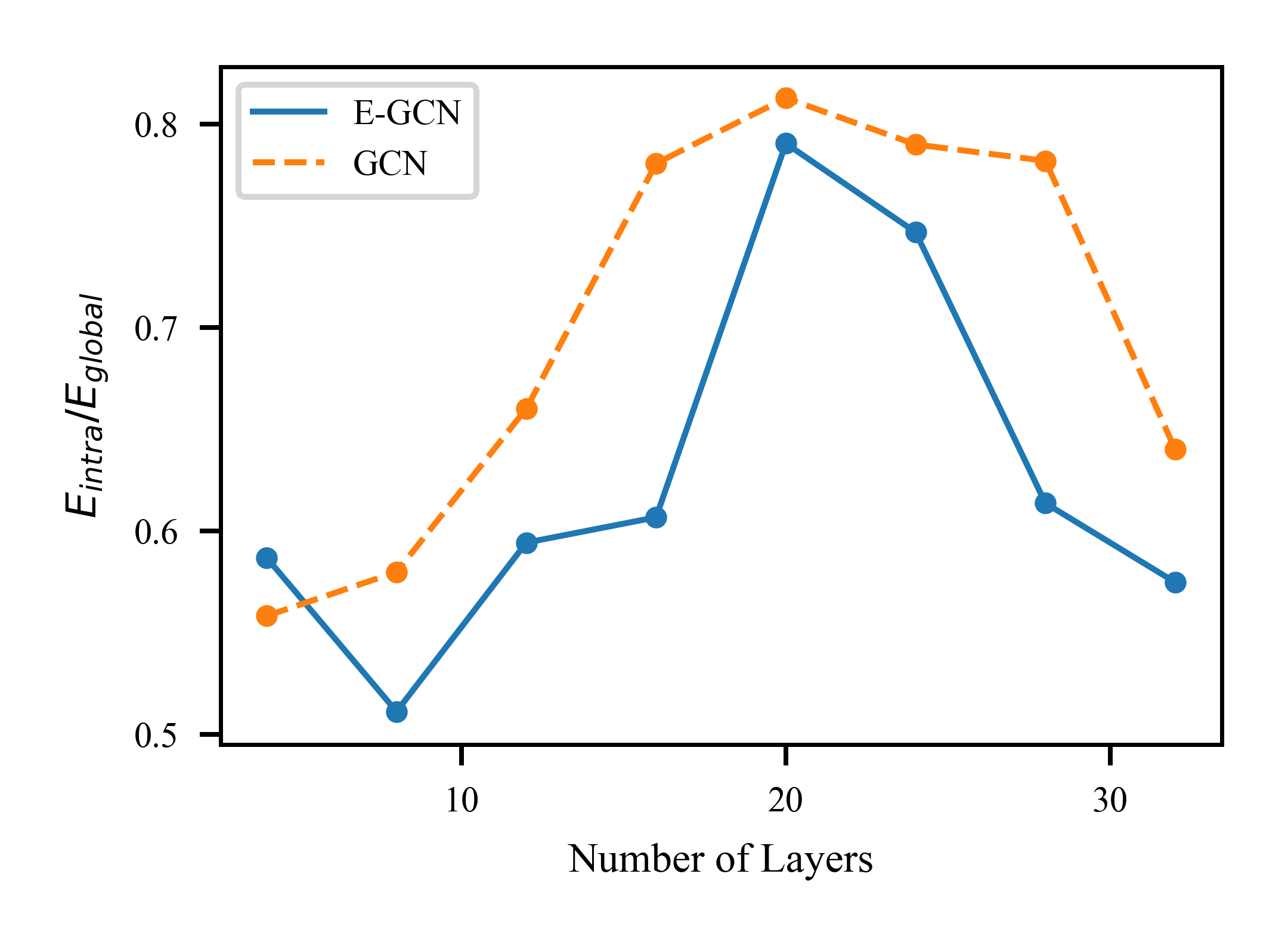}
    \caption{
   The ratio of Intra-Energy to Global Energy for different layers.
    }
    \label{fig:Energy-Layer}
\end{figure}

\subsection{Discussion}
To answer \textbf{Q4}, we perform experiments to further investigate the effectiveness of our loss construction during the training process. 
Limited by space, here we show the results of GCN on Cora. We mainly aim to show the training process rather than reach new SOTAs, so we do not adjust backbone models and hyperparameters. We set the layer to be 16, for the 16-layer GCN suffers from over-smoothing but its accuracy is not too low to show the effectiveness of models. E-GCN denotes our model based on GCN with Inter-Energy Reg and Intra-Energy Reg. We train both models for 1000 epochs, when GCN nearly reaches the early stop point while E-GCN already passed. 

In \autoref{fig:VS}, we show the training curve of 16-layer GCN and E-GCN. 
(a) directly presents the effect of $\mathcal{L}_{intra}$. It shows that adding our regularization to training can keep the local energy at a low rate.

Relating (a) with (b), we can observe that by keeping Intra-Energy low, our model can capture more information from similar nodes and reach high accuracy. 
Moreover, the accuracy curve rises more rapidly than that in GCN. This shows that our regularization can also speed up the training process in the starting few epochs.
As discussed before, our LEReg effectively lifts accuracy of GCN with deep layers. We further explore the relation between accuracy and our LEReg on different layers.
We experiment on GCN with different layers to show the performance of LEReg. We use $E_{intra}/E_G$ to approximate our regularization, where $E_G$ is the global Dirichlet Energy without masking edges. Analyses in \autoref{sec:model} show that we minimize $E_{intra}$ and maximize $E_{inter}$ to optimize the model. And $E_{inter}+E_{intra}\approx E_G$. Therefore, the optimization minimize $E_{intra}/E_G$. Smaller $E_{intra}/E_G$ means lower rate of smoothness.

\autoref{fig:Energy-Layer} shows that for deep layers GCN, LEReg effectively reduce $E_{intra}/E_G$. For GCN of the same layers, our regularization reduces over-smoothing which leads to higher accuracy.

\section{Related Work}

\subsection{Graph Neural Networks}

Motivated by the success of CNNs in computer vision, Bruna et al. \cite{bruna-gcn} and Defferrard et al. \cite{defferrard2016convolutional} develop graph convolution based on spectral graph theory. Afterwards, Kipf et al.  \cite{kipf-gcn} stack layers of first-order Chebyshev polynomial filters \cite{defferrard2016convolutional} with a redefined propagation matrix.
Besides, sampling-based methods have also been developed for fast and scalable GNN training, such as GraphSAGE \cite{sage}, FastGCN \cite{fastGCN}, and AS-GCN \cite{asgcn}. GraphSAGE \cite{sage} was proposed by Hamilton, based on sampling and aggregation for large graphs. 
GAT \cite{gat} utilizes attention mechanism to weigh the importance of neighbors.  MixHop \cite{mixhop}, CSGNN \cite{csgnn}, and ADSF \cite{adsf} share similar thinkings of attention mechanism of adjusting the weights for aggregating.

\subsection{Graph Regularization}
Many effective methods \cite{zhu2003semi,zhou2004learning,ando2007learning} have been proposed for node classification by adding Laplacian regularization to a feature mapping model $f(X)$ to encourage smoothing between connected nodes, where $X \in \mathbb{R}^{N \times F}$ is the node feature matrix. 
These methods only model the features of each node and does not encode the graph structure, and they base on the assumption that connected nodes is likely to be of the same labels.

Based on Laplacian regularization, P-reg\cite{Preg-yang2020rethinking} utilizes label propagation to infuse extra information and changes the edge-centric Laplacian regularization to node-centric. PairNorm\cite{pairnorm} further encourage the similarity between connected nodes and add a negative term to distances between disconnected pairs. To modify the large-scale smoothing caused by the assumption in Laplacian regularization, MADReg \cite{madreg} proposes a smoothing regularization using step size limits to make the graph nodes receive more useful information and less interference noise. 
Combined with data augmentations, BVAT\cite{BVAT} and GRAND \cite{grand} propose to use consistency loss to encourage similarity between different augmentations. GraphMix \cite{graphmix} import MixUp\cite{zhang2017mixup} method to augment graph data. They promote the model to predict the same corresponding labels.

\subsection{Over-smoothing in Deep GNNs}
Dropping nodes or edges are common methods to prevent over-smoothing.
DropEdge \cite{dropedge} points out that for the planar graph and its dual graph, edge deletion in one graph corresponds to the contraction in the other graph and vice versa. GRAND \cite{grand} adopts the way of random dropping nodes to retard the convergence speed of over-smoothing. 
Considering the relationship between GCN and PageRank, Johannes et al. propose APPNP \cite{appnp}. 
Qimai et al. \cite{deeper_insight} apply co-training and self-training to overcome over-smoothing.
DAGNN \cite{dagnn} and SGC \cite{sgc} ascribe over-smoothing to the complexity of deep GCNs, which decouples the representation transformation and propagation to simplify the learning process.
Methods that add skip connections \cite{resgcn,jknet,StrongerGCN,gcnii} aggregate the initial layer or intermediate layers in each hidden layer. 
An extension of vanilla GCN with two techniques: initial residual and identity mapping is proposed in GCNII \cite{gcnii}.
Similarly, MADReg \cite{madreg} optimizes the graph topology based on the prediction result.
Different normalization methods were also proposed to prevent over-smoothing \cite{pairnorm,groupnorm}.

\section{Conclusion}
\label{sec:conclusion}

In this paper, we focus on the issue of adaptive information passing in Graph Neural Networks by regularizing the smoothness of different parts of graphs.
We first observe the smoothness of various classes in the commonly used dataset, Cora, and find that the smoothness among local graphs varies as the structure of the corresponding graphs differs from each other. 
Given the theoretical analysis and empirical study, we prove that there exist differences between local graphs and propose two regularization terms: Inter-Energy Reg and Intra-Energy Reg, based on the measurement of local smoothness. Furthermore, we state that LEReg spreads label information, thus provide more information than global regularization terms. We show that the over-smoothing problem in deep GCNs can also be alleviated with our method. Elaborate experiments show that with our regularization, GNNs achieve better results compared to SOTAs on node classification tasks.

\begin{acks}
This work was supported by the National Natural Science Foundation of China (Grant No.61876006).
\end{acks}

\bibliographystyle{ACM-Reference-Format}
\bibliography{LEReg}

\end{document}